\documentclass[times,review,10pt]{elsarticle}

\usepackage{amssymb}
\usepackage{amsmath}
\usepackage{booktabs}   
\usepackage{colortbl} 
\usepackage{xcolor}    
\usepackage{multirow}
\usepackage{amsmath, amssymb, amsfonts}
\usepackage{bm}
\usepackage{algorithm}
\usepackage{algorithmic}

\usepackage{booktabs}
\usepackage{graphicx}
\usepackage{multirow}
\usepackage{booktabs}
\usepackage{bbding}
\usepackage{amsmath}
\usepackage{amssymb}
\usepackage{amsthm}
\usepackage{bm}
\usepackage{hyperref}

\newcommand{\cca}[1]{\cellcolor{blue!#1!white}}
\newcommand{\ccl}[1]{\cellcolor{orange!#1!white}}

\newtheorem{theorem}{Theorem}
\newtheorem{lemma}{Lemma}

\newtheorem{assumption}{Assumption}

\usepackage{hyperref}
\hypersetup{
    colorlinks=true,
    linkcolor=ccr,
    anchorcolor=ccr,
    citecolor=ccr}

\journal{Knowledge-Based Systems}

\definecolor{highlightblue}{RGB}{235, 245, 255}
\definecolor{lightgrey}{RGB}{240, 240, 240}
\definecolor{posgreen}{RGB}{0, 150, 80}
\definecolor{negred}{RGB}{200, 50, 50}

\begin{document}

\begin{frontmatter}

\title{Divide-and-Conquer Inference for Large-Scale Visual Recognition with Multimodal Large Language Models}

\author[aff1]{Zhipeng Ye} 
\ead{zhipengye@nustti.edu.cn}
\author[aff1]{Jiaqi Huang}
\ead{2207880112@nustti.edu.cn}
\author[aff1]{Feng Jiang\corref{cor1}}
\ead{jf@nustti.edu.cn}
\author[aff2]{Qiufeng Wang}
\ead{qiufeng.wang@xjtlu.edu.cn}
\author[aff2]{Yikang Duan}
\ead{Yikang.Duan23@student.xjtlu.edu.cn}
\author[aff1]{Dawei Wang}
\ead{22017019@nustti.edu.cn}
\author[aff4]{Xihang Zhou}
\ead{xihang.zhou@mail.utoronto.ca}
\author[aff3]{Qian Qiao}
\ead{joeqian@aliyun.com}

\affiliation[aff1]{organization={Taizhou Institute of Science and Technology, Nanjing University of Science and Technology},
            city={Taizhou},
            postcode={225300}, 
            state={Jiangsu},
            country={China}}
\affiliation[aff2]{organization={Department of Intelligence Science, Xi’an Jiaotong-Liverpool University},
            city={Suzhou},
            postcode={215123}, 
            state={Jiangsu},
            country={China}}
\affiliation[aff3]{organization={School of Computer Science and Technology, Soochow University},
            city={Suzhou},
            postcode={215123}, 
            state={Jiangsu},
            country={China}}
\affiliation[aff4]{organization={Department of Statistical Sciences, University of Toronto},
            city={Toronto},
            postcode={M5S 1A1}, 
            state={Ontario},
            country={Canada}}

\cortext[cor1]{Corresponding author}

\begin{abstract}
Multimodal Large Language Models (MLLMs) have demonstrated strong capabilities across a wide range of vision-language tasks. However, when applied to large-scale image classification, their performance degrades significantly as the label space expands—a phenomenon we define as Performance Collapse in Long Sequence Recognition. Through an information-theoretic analysis, we reveal that this collapse stems from a fundamental conflict between the escalating information entropy and the prominent attention dilution and decay within attention mechanisms, which impairs the model's ability to maintain a sufficient signal-to-noise ratio when processing extremely long prompts. To mitigate this, we propose Divide-and-Conquer Inference (DCI), a novel test-time scaling strategy for visual recognition with MLLMs. DCI recursively decomposes complex global classification tasks into multiple simpler, localized sub-problems and employs a dynamic pruning mechanism to compress the search space. This method effectively improves the local signal-to-noise ratio and model accuracy by mitigating the inherent weight dilution issues in long-sequence inference. Moreover, while traditional self-attention incurs a prohibitive quadratic computational complexity, DCI achieves more favorable scaling behavior and substantially accelerates inference in large-scale classification scenarios. Extensive experiments on benchmarks such as ImageNet-1K and ImageNet-21K demonstrate that DCI consistently improves classification accuracy. This enables lightweight open-source models to rival or even surpass frontier closed-source giants without any additional training or fine-tuning. As a model-agnostic, plug-and-play paradigm, DCI offers an efficient approach for scaling the inferential precision of MLLMs in large-scale scenarios. To facilitate reproducibility and further research, the source code is publicly available at: \url{https://github.com/FourierAI/DCI}.
\end{abstract}

\begin{highlights}
\item We identify performance degradation in long-sequence recognition.
\item A Divide-and-Conquer Inference framework is proposed for MLLMs.
\item DCI improves classification accuracy without additional training.
\item DCI reduces inference overhead in large-scale recognition tasks.
\item Extensive experiments validate scalability across diverse MLLMs.
\end{highlights}


\end{frontmatter}



\section{Introduction}

In recent years, Multimodal Large Language Models (MLLMs)\cite{mllm_survey} have emerged as a cornerstone of Artificial Intelligence research. Governed by the Scaling Law\cite{scaling_law}, models such as GPT-4V\cite{gpt4}, Qwen3-VL\cite{qwen3}, and DeepSeek-VL\cite{deepseek_vl} have achieved state-of-the-art performance by scaling up datasets and parameter counts. Beyond synergistic multimodal understanding, these models exhibit profound emergent abilities, including complex reasoning, in-context learning, and self-reflection. Consequently, MLLMs are increasingly viewed as a viable path toward Artificial General Intelligence (AGI)\cite{agi}, sparking a surge of interest in adapting them for traditional Computer Vision (CV) tasks.

Prior research indicates that distilling open-world knowledge from MLLMs into vision models can substantially boost performance. For instance, LLM2CLIP~\cite{llm2clip} replaces the CLIP\cite{clip} text encoder with a pre-trained Large language Model (LLM) and employs lightweight fine-tuning to achieve state-of-the-art (SOTA) results. LaCLIP~\cite{laclip} leverages in-context learning to rewrite image captions into semantically diverse versions, enhancing robustness through ``text-side random augmentation.'' Similarly, IDEA~\cite{idea} utilizes MLLMs to generate fine-grained descriptions, using an Adapter to heighten feature perception in few-shot scenarios. Moreover, LLaVA-OneVision~\cite{llava} unifies diverse vision tasks into a Visual Question Answering (VQA) framework by aligning a frozen SigLIP encoder with the Qwen2\cite{qwen2} language model, significantly improving model versatility.

Despite their success, directly applying MLLMs to large-scale image classification remains challenging. While a prevalent view\cite{bad} attributes this to a lack of domain-specific training data, we argue that the primary bottleneck arises from the growing size of the label space itself. To verify this, we conduct extensive benchmarks on ImageNet-1K using mainstream MLLMs (e.g., LlaMA\cite{llama}, Gemma\cite{gemma}, and Qwen\cite{qwen3}). By sampling label subsets of sizes $\mathcal{S} \in \{10, 20, 100, 200, \allowbreak 500, 1000\}$, we observe a consistent trend (see Fig.~\ref{fig:intro}): \textbf{classification accuracy degrades precipitously as the label set size increases}, regardless of model architecture or scale. We refer to this phenomenon as Performance Collapse in Long Sequence Recognition (PC-LSR). Notably, LlaMA 3.2-11B-Vision maintains 88.60\% accuracy on 20 classes but suffers a catastrophic performance collapse as the set grows to 200, eventually dropping to a mere 0.53\% accuracy on the full 1000-class set.

\begin{figure}[!h]
    \centering
    \includegraphics[width=0.9\linewidth]{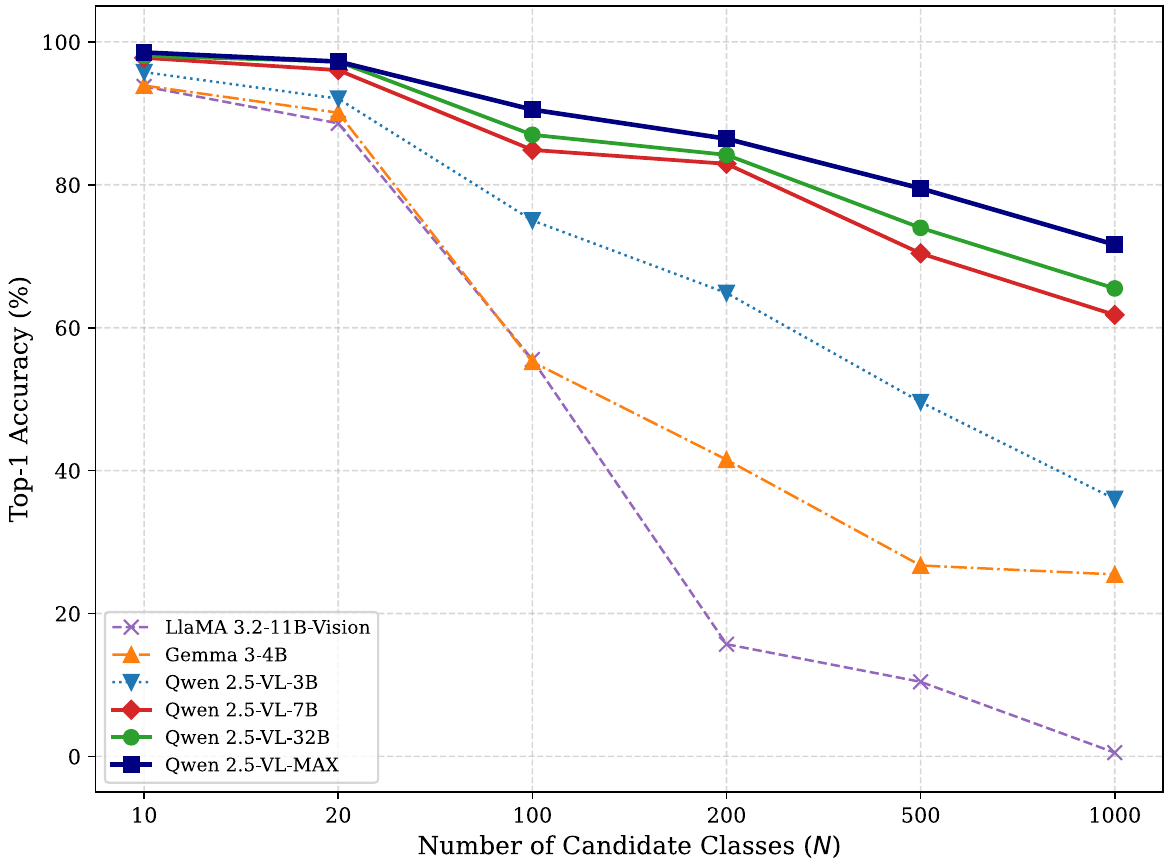}
    \caption{LLM performance on ImageNet-1K across varying numbers of candidate classes. The x-axis represents the specific number of candidate classes presented to the model in a closed-set classification task, sampled at intervals of 10, 20, 100, 200, 500, and 1,000. The y-axis shows the resulting score. Each line represents the performance trajectory of a specific MLLM as the candidate space expands.}
    \label{fig:intro}
\end{figure}

In this paper, we provide an information-theoretic explanation for this phenomenon: \textbf{(1) Uncertainty Scaling}, where increasing label set size raises information entropy and thus decision complexity; and \textbf{(2) Attention Dilution}, where long candidate lists reduce the Signal-to-Noise Ratio (SNR) of attention, impairing the model’s ability to identify the correct label amid excessive textual noise. To mitigate these issues, we propose \textbf{Divide-and-Conquer Inference (DCI)}, a simple yet potent strategy tailored for image classification. Inspired by the divide-and-conquer paradigm, DCI partitions a massive category list into multiple manageable subsets for local inference. By recursively aggregating results, DCI transforms a complex global search into a series of simplified local decisions. This ``coarse-to-fine'' refinement acts as a dynamic attention mask, effectively reducing the influence of irrelevant categories and preserving the SNR. We evaluate DCI across multiple benchmarks (ImageNet-21K\cite{imagenet}, ImageNet-1K\cite{imagenet}, CIFAR-100\cite{cifar100}, CUB-200\cite{cub200}, Food-101\cite{food101}) using various MLLMs. Experimental results show that DCI consistently boosts performance; notably, an open-source Qwen3-VL-8B\cite{qwen3} equipped with DCI can achieve competitive performance compared to trillion-parameter scale models, such as GPT-4\cite{gpt4} and Qwen3-VL-PLUS\cite{qwen3}. Moreover, in large-scale image classification scenarios, while standard flat inference suffers from a prohibitive $\mathcal{O}(N^2)$ computational complexity due to global context bottlenecks, DCI substantially reduces the computational overhead, delivering significantly faster inference speeds.

Our contributions are summarized as follows:
\begin{itemize}
    \item \textbf{Systematic Analysis}: To better understand this behavior, we conduct a systematic empirical study of the performance collapse of MLLMs under large-scale label sets and provide a information-theoretic analysis of its underlying causes.
    \item \textbf{Inference Strategy}: We introduce Divide-and-Conquer Inference (DCI), a novel paradigm that recursively decomposes complex classification tasks into parallel, simple sub-tasks, effectively addressing noise interference in long-context attention.
    \item \textbf{Empirical Validation}: Extensive experiments show that DCI is a training-free, model-agnostic, and plug-and-play framework. In large-scale image classification scenarios, DCI simultaneously delivers higher recognition accuracy and reduced computational latency, allowing lightweight open-source models to rival frontier closed-source giants.
\end{itemize}

\section{Related Work}
\subsection{Test-Time Scaling in Large Language Models}

The Scaling Law\cite{scaling_law} describes the empirical relationship between model performance and scale, including model parameters, data volume, and computational resources. Recent advances in Large Language Models (LLMs), such as GPT-4\cite{gpt4}, DeepSeek-R1\cite{r1}, and Qwen-3\cite{qwen3}, have largely benefited from increasing model size. However, training-time scaling faces growing challenges due to rising computational costs and the limited availability of high-quality data.

To address these limitations, recent studies have shifted toward Test-Time Scaling (TTS)\cite{tts}, which improves downstream performance by allocating additional computation during inference. Existing TTS approaches can be broadly categorized into four types:
(1) \textbf{Parallel scaling} aggregates multiple candidate reasoning paths to improve output reliability, represented by Self-Consistency (SC)\cite{sc}. 
(2) \textbf{Sequential scaling} increases reasoning depth through intermediate steps, such as Chain-of-Thought (CoT)\cite{cot}. 
(3) \textbf{Hybrid scaling} jointly explores reasoning breadth and depth, including Tree of Thoughts (ToT)\cite{tot} and Graph of Thoughts (GoT)\cite{got}. 
(4) \textbf{Internal scaling} enhances autonomous reasoning through reinforcement learning or post-training optimization, as demonstrated by OpenAI-o1\cite{o1} and DeepSeek-R1~\cite{r1}.

In this paper, we propose Divide-and-Conquer Inference (DCI), a novel hybrid scaling strategy for image classification that synergistically combines the strengths of parallel and sequential scaling. On one hand, the parallel component facilitates computational acceleration by decomposing complex tasks into independent sub-tasks for concurrent execution. On the other hand, the sequential mechanism ensures progressive task simplification, where the outputs of each iteration serve as refined inputs for the next, systematically reducing the problem's complexity until the final result is obtained.

\subsection{Divide-and-Conquer Paradigm}

Divide-and-Conquer\cite{dac} is a paradigm for solving complex problems by recursively breaking an expansive problem into several identical sub-problems of smaller scale, until they become simple enough to be solved directly. The solutions to these sub-problems are then combined to solve the original problem\cite{dac2}. As a classic yet powerful strategy, it underpins various domains in computer science. 

In Merge Sort, proposed by von Neumann, an array is halved recursively until each sub-array contains a single element, after which "merge" operations combine these ordered sub-arrays into a larger sorted set. The Fast Fourier Transform (FFT)\cite{fft} is also centered on a divide-and-conquer algorithm. It decomposes a Discrete Fourier Transform (DFT) by odd and even indices into smaller sub-transforms. By exploiting the periodicity and symmetry of complex exponential functions (roots of unity), it reduces the computational complexity from $O(N^2)$ to $O(N \log N)$, making it a revolutionary foundation of digital signal processing.

Inspired by this paradigm, we propose the Divide-and-Conquer Inference (DCI) strategy to address the inherent challenges of MLLMs in image classification. Empirical evidence suggests that while MLMs struggle with recognition accuracy when handling long sequences of categories, they maintain robust performance on shorter sequences. The root cause of this degradation lies in the attention mechanism, where an excessive number of categories introduces significant distractors that impair effective feature extraction.

Through the DCI strategy, we deconstruct the complex, long-sequence classification task into a series of sub-sequence recognition problems. This transformation constrains the model’s focus to a limited set of candidates at each stage, thereby effectively suppressing noise interference within the attention mechanism and enhancing inference robustness. These sub-tasks are solved independently, with their outputs serving as inputs for the subsequent round of inference. As the recursion progresses, the problem scale and difficulty diminish, while each sub-task remains structurally isomorphic to the original problem. Moreover, through this grouping-based divide-and-conquer mechanism, DCI dramatically undercuts the standard $\mathcal{O}(N^2)$ flat inference baseline. This enables our approach to achieve highly accelerated inference speeds in extreme-scale image classification scenarios, particularly on massive benchmarks like ImageNet-21K\cite{imagenet}.

\section{Theoretical Analysis of the PC-LSR}

To understand the root causes of the empirical performance collapse observed in our benchmarks, we provide a formal information-theoretic analysis in this section. By modeling the MLLM inference process as an information-constrained channel, we aim to demonstrate that the observed degradation stems from a fundamental architectural bottleneck—the intrinsic inability of the attention-based flat inference paradigm—which performs classification by considering all $N$ candidate labels within a single forward pass—to sustain information density as the label space scales.

\subsection{Preliminaries and Problem Setup}
We model the MLLM based image classification as an information-constrained Markov chain: $Y \to X \to \hat{Y}$, where:
\begin{itemize}
    \item $Y \in \mathcal{Y}$ is the ground truth label with a label space size $|\mathcal{Y}|=K$.
    \item $X$ represents the visual features extracted from the input image.
    \item $\hat{Y} \in \mathcal{Y}$ is the model's prediction given a specific prompt $\mathcal{P}_K$.
\end{itemize}
Our primary objective is to analyze the asymptotic behavior of the classification error $P_e = P(\hat{Y} \neq Y)$ as the number of candidates $K$ scales toward infinity.

\subsection{Foundational Assumptions and Lemmas}

To characterize the gap between the information required for task resolution and the information supplied by the attention-based architecture, we introduce the following assumptions.

\begin{assumption}[Dispersed Label Distribution]
\label{asmp:distribution}
The label distribution is non-degenerate. Specifically, there exists a constant $\beta \geq 1$ such that for all $k \in \mathcal{Y}$, the prior probability satisfies $p_k \leq \frac{\beta}{K}$. This ensures that the task difficulty scales with the size of the candidate set. All known datasets satisfy this assumption.
\end{assumption}

\begin{lemma}[Uncertainty Scaling]
\label{lem:entropy}
Under Assumption \ref{asmp:distribution}, the label entropy $H(Y)$, which represents the theoretical information demand, scales logarithmically with the size of the label space $K$:
\begin{equation}\label{eq:lemma1}
    H(Y) \geq \log K - \log \beta
\end{equation}
\end{lemma}

\begin{proof}
By the definition of Shannon entropy:
\begin{align}\label{eq:lemma1_proof}
    H(Y) &= \sum_{k=1}^K p_k \log \frac{1}{p_k} \nonumber \\
         &\geq \sum_{k=1}^K p_k \log \left( \frac{K}{\beta} \right) \quad (\text{since } p_k \leq \beta/K \implies 1/p_k \geq K/\beta) \nonumber \\
         &= (\log K - \log \beta) \sum_{k=1}^K p_k = \log K - \log \beta
\end{align}
In the special case of a \textbf{Uniform Distribution} ($p_k = 1/K$), the entropy reaches its maximum $H(Y) = \log K$.
\end{proof}

While Lemma \ref{lem:entropy} quantifies the increasing difficulty of the task as the label space expands, it only accounts for the 'demand' side of the information equation. The following lemma shifts focus to the 'supply' side, analyzing how the transformer's core mechanism-attention—behaves when tasked with processing an extensive number of candidates.

\begin{lemma}[Attention Dilution]
\label{lem:dilution}
In the attention mechanism, let $\alpha_k$ be the attention weight assigned to the $k$-th candidate. For analytical tractability, we analyze the scaling behavior of Softmax-based attention under a bounded and non-dominant logit regime. As the number of competing candidates increases, the normalization in Softmax spreads the probability mass across all candidates, resulting in an $O(1/K)$ decay in the expected attention assigned to each individual candidate. Consequently, the effective mutual information $I(Y; \hat{Y})$ is bounded by a constant $I_{\max}$ representing the model's fixed ``attention bandwidth'':
\begin{equation}\label{eq:lemma2}
    I(Y; \hat{Y}) \leq I_{\max} < \infty, \quad \forall K \gg 1
\end{equation}
\end{lemma}

\begin{proof}
Let the attention weight be defined by the Softmax function 
\begin{equation}\label{eq:softmax}
\alpha_k = \frac{\exp(e_k)}{  \sum_{j=1}^K \exp(e_j)}
\end{equation}
Let $Z_j = \exp(e_j)$ be the transformed random variables. Assume $\{Z_j\}_{j=1}^K$ are random variables with common finite mean $E[Z_j] = \mu > 0$ and finite variance $\mathrm{Var}(Z_j) = \sigma^2 < \infty$.

According to the \textit{Kolmogorov Strong Law of Large Numbers} (LLN), as $K \to \infty$, the denominator satisfies:
\begin{equation}\label{eq:lemma2_mean1}
    \frac{1}{K} \sum_{j=1}^K Z_j \xrightarrow{a.s.} \mu
\end{equation}
Applying the first-order Taylor expansion (Delta Method) for the expectation of the ratio, we have:
\begin{equation}\label{eq:lemma2_mean2}
    E[\alpha_k] = E\left[ \frac{Z_k}{\sum_{j=1}^K Z_j} \right] \approx \frac{E[Z_k]}{K \cdot \mu} = \frac{\mu}{K\mu} = \frac{1}{K}
\end{equation}
This confirms that $E[\alpha_k] = O(1/K)$. In this information-constrained channel, we define the \textit{Effective Attention SNR} as the ratio of the expected signal power to the expected noise power:
\begin{equation}\label{eq:snr_refined}
    \text{SNR}_k \triangleq \frac{E[\alpha_k]}{E[\sum_{j \neq k} \alpha_j]} = \frac{1/K}{1 - 1/K} = \frac{1}{K-1}
\end{equation}
As $K \to \infty$, $\text{SNR}_k$ vanishes. Inspired by Shannon-Hartley Theorem\cite{shannon}, the channel capacity $C$ (the upper bound of $I(Y; \hat{Y})$) is given by $W \log_2(1 + \text{SNR})$, where $W$ is a constant representing the model's fixed attention bandwidth (e.g., embedding dimension and heads). Substituting $\text{SNR}_k$:
\begin{equation}\label{eq:shannon_bound}
    I(Y; \hat{Y}) \leq W \log_2\left(1 + \frac{1}{K-1}\right)
\end{equation}
The result implies that as the label space $K$ scales, the extractable mutual information $I(Y; \hat{Y})$ fails to sustain the density required for accurate classification, eventually saturating at a constant $I_{\max}$.
\end{proof}

\subsection{Main Result: The Information Deficit Theorem}

\begin{theorem}[Information Deficit Theorem]
\label{thm:main}
In a flat label space of size $K$, the classification error $P_e(K)$ of an MLLM is lower-bounded by:
\begin{equation}
    P_e(K) \geq 1 - \frac{I_{\max} + \log (2\beta)}{\log K}
\label{eq:final_bound}
\end{equation}
Consequently, the error probability tends to increase asymptotically as the label space expands.
\end{theorem}

\begin{proof}
We begin with \textbf{Fano's Inequality\cite{fano}}, which relates the conditional entropy $H(Y|\hat{Y})$ to the error probability $P_e$:
\begin{equation}\label{eq:result_proof1}
    H(Y|\hat{Y}) \leq h(P_e) + P_e \log(K-1)
\end{equation}
where $h(P_e) \leq \log 2$ is the binary entropy. Using the identity $H(Y|\hat{Y}) = H(Y) - I(Y; \hat{Y})$ and the approximation $\log(K-1) < \log K$:
\begin{equation}\label{eq:result_proof2}
    P_e \geq \frac{H(Y) - I(Y; \hat{Y}) - \log 2}{\log K}
\end{equation}

\textbf{Step 1: Substituting Demand and Supply.} Applying the results from Lemma \ref{lem:entropy} (Information Demand) and Lemma \ref{lem:dilution} (Information Supply):
\begin{align}\label{eq:result_proof3}
    P_e &\geq \frac{(\log K - \log \beta) - I_{\max} - \log 2}{\log K} \nonumber \\
        &= \frac{\log K - (I_{\max} + \log \beta + \log 2)}{\log K} \nonumber \\
        &= 1 - \frac{I_{\max} + \log (2\beta)}{\log K}
\end{align}

\textbf{Step 2: Asymptotic Analysis.} 
As $K \to \infty$, the constant term $C = I_{\max} + \log (2\beta)$ remains fixed. Thus:
\begin{equation}\label{eq:result_proof4}
    \lim_{K \to \infty} P_e(K) \geq \lim_{K \to \infty} \left( 1 - \frac{C}{\log K} \right) = 1
\end{equation}
Since $P_e \leq 1$ by definition, we conclude that the error probability $P_e$ approaches 1 as $K$ increases.
\end{proof}

\subsection{Discussion}
The Information Deficit Theorem (Theorem \ref{thm:main}) formalizes the fundamental collapse of flat classification paradigms in large-scale MLLM inference. This collapse is driven by a widening gap between the task's \textit{informational demand} and the model's \textit{architectural supply}:

\begin{enumerate}

\item \textbf{The Log-Linear Imbalance}: As the candidate set size $K$ increases, the information required to resolve label uncertainty grows logarithmically ($\log K$). In contrast, the effective information extracted by the model diminishes due to the inherent $O(1/K)$ attention dilution effect, creating an increasing gap between informational demand and model capacity.

\item \textbf{The Saturation of Discriminability}: The fixed attention capacity $I_{\max}$ imposes an upper bound on the model's discriminative ability. Once the entropy requirement exceeds this limit, the signal of the correct candidate is overwhelmed by distractors, causing the decision process to degenerate toward random guessing.

\end{enumerate}
In summary, Theorem \ref{thm:main} explains why standard MLLM prompting fails in large-scale classification tasks. The identified information deficit indicates that single-step global inference faces inherent limitations as the label space expands, motivating a shift from flat inference to structured approaches such as our proposed Divide-and-Conquer Inference (DCI) framework.

\section{Methodology}

Existing studies mainly improve classification performance by scaling model size or optimizing architectures. However, such approaches are often limited by high computational costs, extensive training data requirements, and complex optimization processes. In contrast, we propose Divide-and-Conquer Inference (DCI), a training-free framework that achieves test-time scaling by reducing the candidate space size $K$. Unlike sequential strategies such as Chain-of-Thought (CoT), DCI adopts a hybrid scaling paradigm tailored for vision tasks, alleviating perceptual burden through spatial decomposition.

\begin{figure}[h!]
    \centering
    \includegraphics[width=0.99\linewidth]{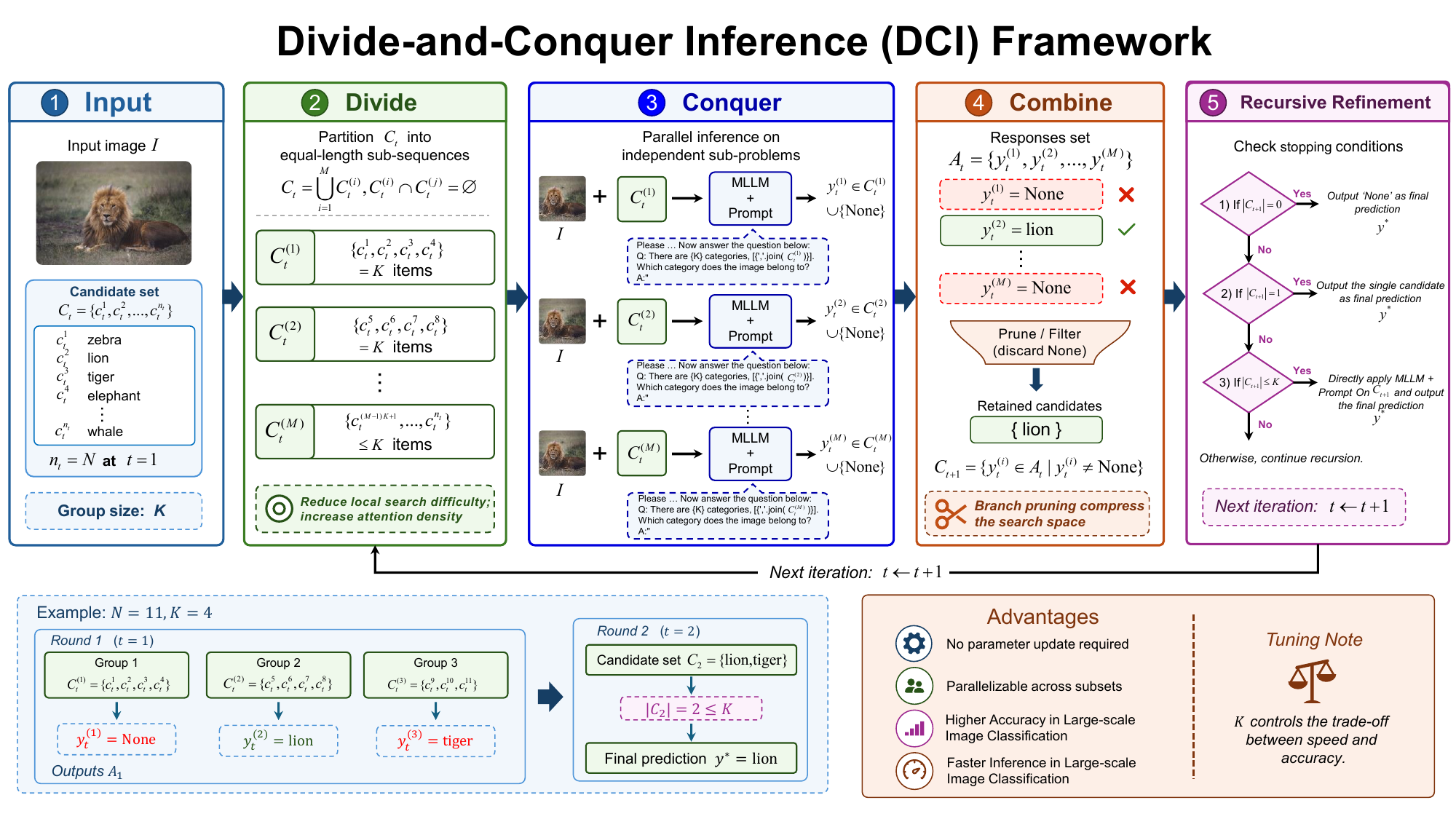}
    \caption{Overview of the Divide-and-Conquer Inference (DCI) framework. The process involves three main stages: (1) Divide: partitioning the candidate space into smaller subsets to increase local attention density; (2) Conquer: performing parallel MLLM inference on each subset; and (3) Combine: aggregating results and applying recursive branch pruning to narrow down the search space.}
    \label{fig:method}
\end{figure}

The necessity of DCI is rooted in the Information Deficit Theorem (Theorem 1). As the category scale $N$ increases, the classification task transitions from a Simple Problem ($q_{\text{simple}}$) to a Hard Problem ($q_{\text{hard}}$), where classification boundaries become blurred and overcrowded. This collapse is mathematically explained by Lemma \ref{lem:dilution} (Statistical Weight Dilution). To bridge this theoretical gap, the DCI framework recursively decomposes $q_{\text{hard}}$ into a set of isomorphic $q_{\text{simple}}$ instances. By systematically compressing the search space, the method restores the signal-to-noise ratio in each local inference pass. The execution cycle consists of three core phases: \textbf{Divide}, \textbf{Conquer}, and \textbf{Combine}. Fig.\ref{fig:method} illustrates the overall procedure and underlying principles of our proposed DCI method. The following sections detail the mathematical mechanisms and the implementation workflow of these stages.

\subsection{Divide Phase}

As illustrated in the Fig.\ref{fig:method}, during the Divide phase, the category set of the current iteration is partitioned into several sub-sequences of equal length. Let $K$ denote the length of each sub-sequence. We define $t \in \{1, 2, \dots, T\}$ as the current iteration index, where the maximum number of iterations $T$ satisfies $1 \leq T \leq \lceil \log_{K} N \rceil$. The input category sequence for the current iteration is denoted as $C_t = \{c_t^1, c_t^2, \dots, c_t^{n_t}\}$, where $n_t$ is the length of the input sequence at iteration $t$; specifically, $n_t=N$ when $t=1$.

When $n_t$ is divisible by $K$ (i.e., $n_t \bmod K = 0$), the sequence is uniformly partitioned into $M = \frac{n_t}{K}$ sub-sequences, with the $i$-th sub-sequence expressed as:
\begin{equation}
    C_t^{(i)} = (c_t^{(i-1)K+1}, c_t^{(i-1)K+2}, \dots, c_t^{iK}), \quad i = 1, 2, \dots, M
\end{equation}
In cases where $n$ is not divisible by $K$ (i.e., $n \bmod K \ne 0$), the total number of partitions is $M = \lceil \frac{n_t}{K} \rceil$, and the resulting sub-sequences are defined as:
\begin{equation}
    C_t^{(i)} =
    \begin{cases}
    (c_t^{(i-1)K + 1}, c_t^{(i-1)K + 2}, \dots, c_t^{iK}), & 1 \leq i \leq M-1 \\[6pt]
    (c_t^{(i-1)K + 1}, \dots, c_t^{n_t}), & i = M
    \end{cases}
\end{equation}
In summary, the collection of partitioned sub-sequences satisfies the following conditions:
\begin{equation}
    C_t = \bigcup_{i=1}^{M} C_t^{(i)}, \quad C_t^{(i)} \cap C_t^{(j)} = \varnothing \quad (i \ne j)
\end{equation}

In the Divide phase, by performing an equivalent partition of the original category space, we aim to counteract the dilution effect of attention weights and maximize local attention density. This strategy ensures that the input information aligns with the model's optimal "attention bandwidth."

\subsection{Conquer Phase}

\begin{figure}[h!]
    \centering
    \includegraphics[width=0.99\linewidth]{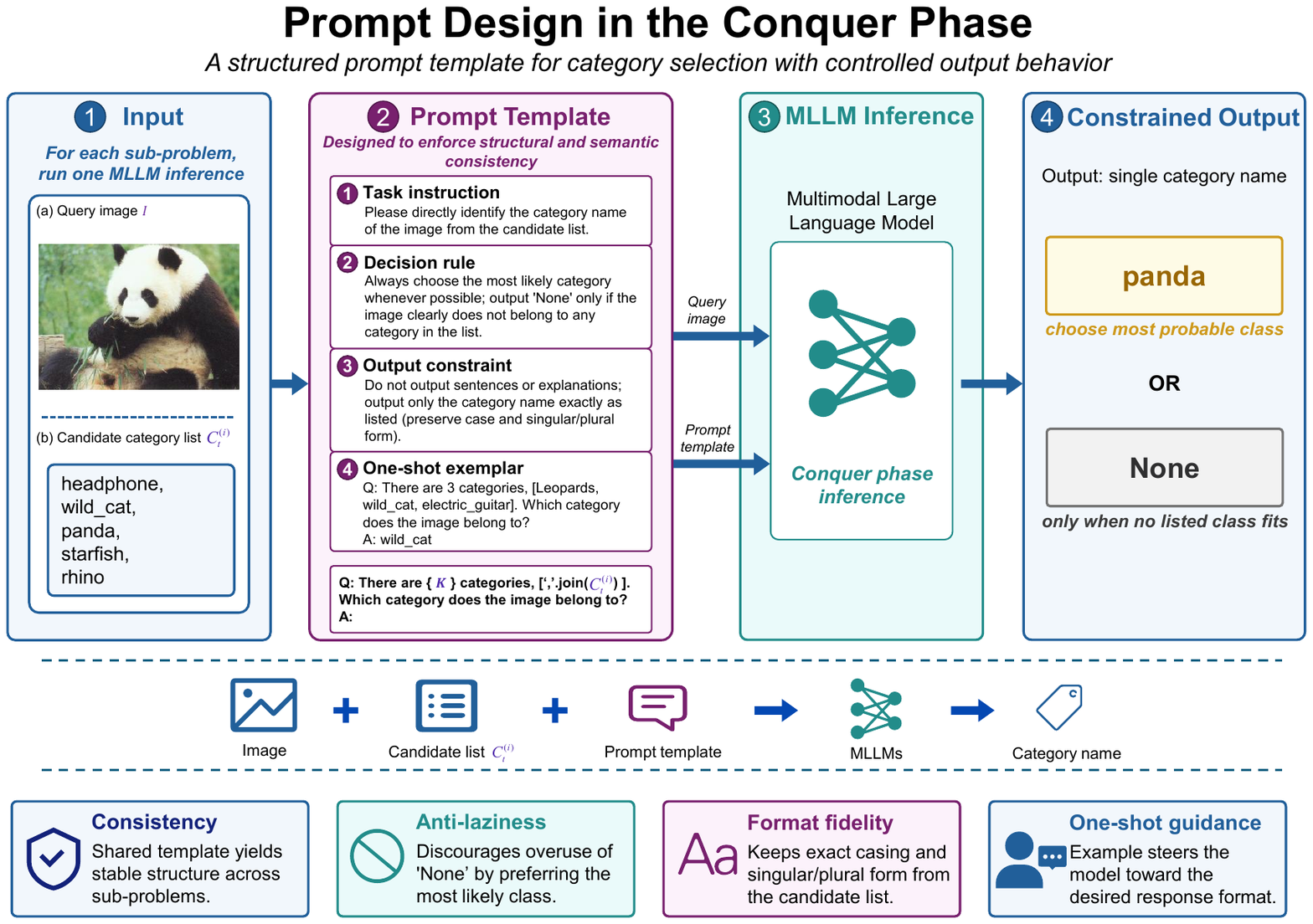}
    \caption{Overview of the Conquer Phase inference workflow. A structured prompt template, incorporating decision rules and one-shot guidance, is used to steer the MLLM toward generating a single, format-constrained category name from a candidate list.}
    \label{fig:conquer}
\end{figure}

In the Conquer phase, a MLLM inference is executed for each sub-problem. We employ a prompt template $f_{prompt}$ designed to enforce structural and semantic consistency:

Fig.\ref{fig:conquer} illustrates several strategies for prompt template design to ensure controllable output behavior from LLMs. Within this template, if the correct category is absent from the sub-sequence $C_t^{(i)}$, the model is instructed to output ``None.'' To mitigate ``model laziness''---a phenomenon where the model over-relies on the null option---the prompt encourages identifying the most probable object rather than defaulting to ``None.'' Furthermore, strict constraints are imposed on the output format to ensure that the casing and plurality of category names remain identical to the input. Finally, a one-shot exemplar is provided to guide the model toward the desired response format.

Based on the partitioned sub-sequences, we construct a set of sub-queries $\mathcal{Q}_t = \{ q_t^{(1)}, q_t^{(2)}, \dots, q_t^{(M)}\}$. Each sub-query $q_t^{(i)}$ is formulated by applying the prompt template $f_{prompt}$ to the corresponding category subset $C_t^{(i)}$:
\begin{equation}
    q_t^{(i)} = f_{prompt}(C_t^{(i)})
\end{equation}
Subsequently, the image $I$ and the sub-query $q_t^{(i)}$ are fed into the MLLM, parameterized by $\theta$, to obtain the response:
\begin{equation}
    y_t^{(i)} = \mathcal{M}_{\text{MLLM}}(I, q_t^{(i)}; \theta)
\end{equation}
As dictated by the prompt design, the output space for $y_t^{(i)}$ is constrained such that:
\begin{equation}
    y_t^{(i)} \in C_t^{(i)} \cup \{ \text{``None''} \}
\end{equation}

For multimodal large models, the sub-query $q_t^{(i)}$ represents a significantly lower cognitive load compared to the original hard problem $q_{\text{hard}}$. Consequently, DCI effectively decomposes a complex classification task into multiple $q_{\text{simple}}$ instances. In addition, since these sub-problems are mutually independent, they can be processed via concurrent and parallel computing, which substantially optimizes inference efficiency, as detailed in the Algorithm~\ref{alg:D&CI}.

\subsection{Combine Phase}

Upon the completion of the parallel inference in the Conquer stage, we obtain a response set $\mathcal{A}_t = \{ y_t^{(1)}, y_t^{(2)}, \dots, y_t^{(M)} \}$. Since the majority of sub-sequences $C_t^{(i)}$ do not contain the ground-truth category, a large proportion of elements in $\mathcal{A}_t$ are "None," representing invalid search branches. To facilitate an efficient recursive search, the Combine phase performs Branch Pruning by filtering out these null responses. This process isolates potential target categories to construct the input set for the subsequent iteration $C_{t+1}$:

\begin{equation}
C_{t+1} = \{ y_t^{(i)} \in \mathcal{A}_t \mid y_t^{(i)} \neq \text{None} \}
\end{equation}
A special case occurs when $\mathcal{A}_t$ consists entirely of "None," resulting in an empty set $C_{t+1} = \varnothing$ and a final negative prediction. This outcome typically stems from the absence of the target object or an identification failure by the MLLM, leading to a misclassification.

This mechanism significantly compresses the search space, offering two primary advantages:
\begin{itemize}
    \item \textbf{Spatial Sparsification:} By eliminating redundant branches, the initially dense global search space is transformed into a sparse, highly focused set of candidates.
    \item \textbf{Computational Reduction:} The number of sub-problems generated in subsequent recursive rounds is substantially decreased, ensuring that computational resources are concentrated on high-confidence regions.
\end{itemize}

The DCI framework continues this cycle of partitioning and pruning until a stable prediction is reached or the candidate set converges, effectively navigating the trade-off between search breadth and inferential precision.

\subsection{Recursive Hierarchy and Termination Guarantees}

\begin{algorithm}[!ht]
\caption{Divide-and-Conquer Inference (DCI)}
\label{alg:D&CI}
\begin{algorithmic}[1]
\renewcommand{\algorithmicrequire}{\textbf{Input:}}
\renewcommand{\algorithmicensure}{\textbf{Output:}}
\REQUIRE Image $I$, Group size $K$, Candidate label set $C_{in}$, Model $\mathcal{M}_{\text{MLLM}}$
\ENSURE Final predicted category $y^*$

\STATE \textbf{Function} D\&C\_Inference($I, K, C_t$)
    \IF{$|C_t| \leq K$} 
        \RETURN $y^* \leftarrow \mathcal{M}_{\text{MLLM}}(I, f_{prompt}(C_t))$
    \ENDIF

    \STATE \COMMENT{\textbf{Divide}: Partition $C_t$ into $M = \lceil |C_t|/K \rceil$ disjoint subsets}
    \STATE $\mathcal{S}_t \leftarrow \{C_t^{(1)}, C_t^{(2)}, \dots, C_t^{(M)}\}$ such that $\bigcup_{i=1}^M C_t^{(i)} = C_t$ and $|C_t^{(i)}| \leq K$
    
    \STATE $\mathcal{A}_t \leftarrow \emptyset$
    \FOR{\textbf{each} $C_t^{(i)} \in \mathcal{S}_t$ \textbf{in parallel}}
        \STATE $y_t^{(i)} \leftarrow \text{D\&C\_Inference}(I, K, C_t^{(i)})$
        \STATE $\mathcal{A}_t \leftarrow \mathcal{A}_t \cup \{y_t^{(i)}\}$
    \ENDFOR

    \STATE $C_{t+1} \leftarrow \{ y \in \mathcal{A}_t \mid y \neq \text{None} \}$

    \IF{$|C_{t+1}| = 0$}
        \RETURN None
    \ELSIF{$|C_{t+1}| = 1$}
        \RETURN $C_{t+1}[0]$
    \ELSE
        \RETURN D\&C\_Inference($I, K, C_{t+1}$)
    \ENDIF
\end{algorithmic}
\end{algorithm}

Upon obtaining the refined candidate set $C_{t+1}$, it serves as the input sequence for the subsequent iteration, where the \textit{Divide}, \textit{Conquer}, and \textit{Combine} phases are executed repeatedly. As the recursion progresses, the candidate set monotonically shrinks until its cardinality $|C_{t+1}| \leq K$, at which point the model directly outputs the final prediction.
The termination of DCI is guaranteed by two factors:

\begin{itemize}
    \item \textbf{Candidate Reduction}: Due to the sparsity of ground-truth labels, the filtered candidate set progressively shrinks ($|C_{t+1}| < |C_t|$). The zero-shot capability of MLLMs further enables branch pruning through "None" responses.

    \item \textbf{Bounded Recursion Depth}: Even in the worst case, the problem size decreases at a rate of $1/K$, yielding a maximum recursion depth of $O(\log_K N)$ and ensuring eventual termination.
\end{itemize}

\subsection{Computational Efficiency and Complexity Advantage}\label{sec:ceaca}

To evaluate the efficiency and computational overhead of the Divide-and-Conquer Inference (DCI) framework, we provide a rigorous mathematical analysis. Let $N$ be the total number of categories, $K$ be the group size, and $m = \lceil \log_K N \rceil$ be the maximum number of iterations. We denote $L(K)$ as the computational latency of a single MLLM inference pass for a local group of size $K$.

The time complexity $T(N, K)$ satisfies the following recurrence relation:
\begin{equation}
T(N,K) \leq 
\begin{cases}
L(K), & \text{if } N \leq K \\
\left\lceil \frac{N}{K} \right\rceil \cdot L(K) + T\left(\left\lceil \frac{N}{K} \right\rceil\right), & \text{if } N > K
\end{cases}
\end{equation}

For the general case where $N > K$, the total complexity can be derived through iterative substitution across all hierarchical layers:
\begin{equation}\label{eq:complexity1}
\begin{aligned}
T(N,K) &\leq \sum_{i=1}^{m} \left\lceil \frac{N}{K^i} \right\rceil \cdot L(K) \\
& \leq \sum_{i=1}^{m} \left( \frac{N}{K^i} + 1\right) \cdot L(K) \\
& = \left( \frac{N(K^m-1)}{(K-1)K^m} + m\right) \cdot L(K) \\
& \leq \left( \frac{N}{K-1} + \lceil \log_K N \rceil \right) \cdot L(K)
\end{aligned}
\end{equation}

Consequently, in the worst-case scenario, the total number of MLLM calls required for a single DCI pass does not exceed $\frac{N}{K-1} + \lceil \log_K N \rceil$, whereas the best-case scenario under ideal pruning requires only $\lceil \frac{N}{K} \rceil$ calls. Crucially, in modern Transformer-based MLLMs, the self-attention mechanism incurs a quadratic overhead relative to the context length, meaning the local inference cost scales as $L(K) \sim \mathcal{O}(K^2)$. Substituting this characteristic into our formulation yields the total asymptotic complexity of DCI as a bivariate function of $N$ and $K$ (where $K \leq N$):
\begin{equation}\label{eq:complexity2}
T(N, K) \approx \left( \frac{N}{K} + \log_K N - 1 \right) \cdot K^2
\end{equation}

Geometrically, this expression defines a continuous theoretical surface within a three-dimensional space. Fig.~\ref{fig:complexity}(a) visualizes this 3D global complexity landscape, illustrating how the joint interaction of $N$ and $K$ shapes the overall computational overhead. To better examine the scalable properties, Fig.~\ref{fig:complexity}(b) extracts the 2D scalability profiles as a function of $K$ across discrete data scales ($N \in \{100, 1000, 20000, 50000\}$).

At smaller scales ($N = 100, 1000$), the computational cost monotonically decreases as $K$ increases. Within this regime, DCI behaves as a standard Test-Time Scaling (TTS) strategy, allowing the framework to flexibly trade test-time computation for higher recognition accuracy. However, as the category space scales up to extreme dimensions ($N = 20000, 50000$), the complexity profiles transition into a distinct "U-shaped" topology. In these large-scale scenarios, the total cost no longer decreases monotonically with larger group sizes, as the quadratic inflation of the local cost $L(K)$ eventually dominates the reduction in recursive iterations.

For standard flat inference, the computational complexity scales as $\mathcal{O}(N^2)$ due to the inherent context bottlenecks of global self-attention mechanisms. By plotting these $\mathcal{O}(N^2)$ baselines alongside our profiles in Fig.~\ref{fig:complexity}, a critical crossover effect emerges. For massive category spaces (e.g., extreme-scale benchmarks like ImageNet-21K), DCI substantially undercuts the flat inference baseline within a well-defined interval of $K$. \textbf{In this optimal range, DCI simultaneously delivers superior classification accuracy and reduced computational latency over flat inference.} We define this highly efficient operational window as the \textbf{DCI advantage region}, a fundamental theoretical property that will be consistently mirrored in our subsequent empirical validations.

\begin{figure}[!ht]
    \centering
    \includegraphics[width=0.99\linewidth]{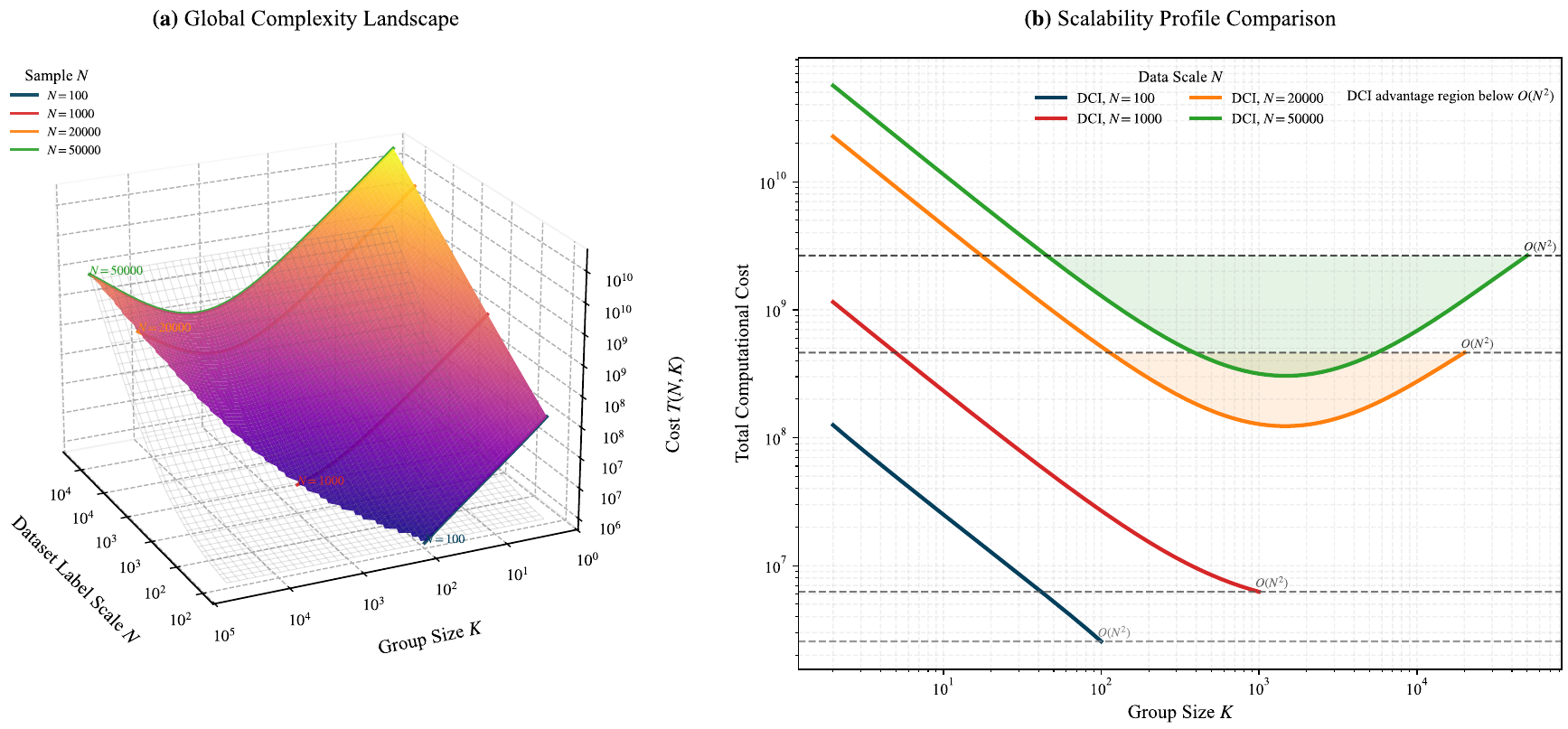}
    \caption{Complexity and scalability analysis. (a) Global Complexity Landscape: Total cost $T(N, K)$ as a function of data scale $N$ and group size $K$. (b) Scalability Profiles Comparision: Comparison of DCI costs across different $N$, highlighting the "U-shaped" curves and the DCI advantage region where costs remain significantly below the $\mathcal{O}(N^2)$ baseline.}
    \label{fig:complexity}
\end{figure}

\section{Experiment}

In this section, we conduct comprehensive experiments to evaluate the proposed DCI framework in terms of effectiveness, scalability, and computational efficiency. We first introduce the experimental settings, including datasets, models, and implementation details. Next, we compare DCI with baseline methods across diverse MLLM architectures and benchmark datasets to validate its overall effectiveness. We then investigate the performance of DCI in large-scale image classification scenarios with up to 21,000 categories, focusing on its robustness against PC-LSR. Finally, ablation studies are conducted to analyze the impact of key hyperparameters and grouping strategies on model performance.

\subsection{Experimental Setup}
All experiments run on a server equipped with an AMD EPYC 7642 processor, four NVIDIA RTX 4090 graphics cards, 256GB DDR4 memory, a 2TB solid-state drive, and an 8TB mechanical hard drive. The server environment comprises Ubuntu 22.04.3 LTS as the operating system, Python 3.10, PyTorch 2.8.0, and Hugging Face Transformers 4.57.3. Since we deploy LLM inference services locally, we utilize the vLLM framework (version 0.15.0). All multimodal large models adopt 32-bit floating-point precision and default temperature coefficients and sampling strategies. 

\subsection{Effectiveness of Divide-and-Conquer Inference Across Diverse Benchmarks}
To evaluate the efficacy of our proposed Divide-and-Conquer Inference (DCI) strategy, we conduct extensive experiments across four benchmark datasets. These range from moderate-scale datasets (e.g., CIFAR-100\cite{cifar100}, CUB-200\cite{cub200}, Food-101 \cite{food101}) to large-scale category spaces (ImageNet-1K\cite{imagenet}). To verify the universality of our framework, we select a diverse suite of evaluation models, including open-source MLLMs such as Gemma3-4b\cite{gemma}, GLM-4V-9B\cite{glm}, and LLaMA-3.2-11B\cite{llama}, as well as leading closed-source models like Qwen3-VL-Plus\cite{qwen3} and GPT-4\cite{gpt4}. Furthermore, we perform comprehensive validation across the popular Qwen series (including Qwen2.5-VL\cite{qwen2} and Qwen3-VL\cite{qwen3}) to demonstrate the consistent benefits of DCI. All experiments adopt a fixed group size of $K=10$ for DCI.
\begin{table}[htb!]
  \centering
    \caption{Performance comparison of various Multimodal Large Language Models (MLLMs) with and without the proposed Divide-and-Conquer Inference (DCI) strategy. Accuracy (\%) is reported across multiple benchmarks. $\Delta$ represents the performance gain achieved by DCI. Bold values indicate the best performance within each column.}
  \label{tab:benchmark}
  \small
  \renewcommand{\arraystretch}{1.2} 
    \resizebox{\linewidth}{!}{%
  \begin{tabular}{l cccc | c}
    \toprule
     Model & \textbf{ImageNet-1K\cite{imagenet}} & \textbf{CIFAR-100\cite{cifar100}} & \textbf{CUB-200\cite{cub200}} & \textbf{Food-101\cite{food101}} & \textbf{Mean} \\ 
    \midrule 

    Gemma3-4b\cite{gemma}   & 25.47 & 54.93  & 13.50 & 70.99 & 41.22 \\
    GLM-4V-9B\cite{glm}    & 22.18 & 56.75  &  5.83 & 17.92 & 25.67 \\
    LLaMA-3.2-11B\cite{llama} &  0.53 & 30.81  &  0.33 & 47.23 & 19.73 \\
    
    \rowcolor{lightgrey} \textcolor{gray}{Qwen3-VL-Plus\cite{qwen3}} & \textcolor{gray}{71.44} & \textcolor{gray}{\textbf{82.81}} &  \textcolor{gray}{63.83} & \textcolor{gray}{\textbf{88.71}} & \textcolor{gray}{\textbf{76.70}} \\
    \rowcolor{lightgrey} \textcolor{gray}{GPT-4(v4.1)\cite{gpt4}} & \textcolor{gray}{62.63} & \textcolor{gray}{64.46}  & \textcolor{gray}{\textbf{73.83}} & \textcolor{gray}{80.79} & \textcolor{gray}{70.43} \\

    \midrule
    Qwen2.5-VL-7B\cite{qwen2} & 61.80 & 61.79 & 35.29 & 79.39 & 59.57 \\
    \rowcolor{highlightblue} w/ \textbf{DCI} & \textbf{71.57} & 65.05  & 48.60 & 80.76 & 66.50 \\
    \rowcolor{highlightblue} $\Delta$ & \textcolor{posgreen}{+9.77} & \textcolor{posgreen}{+3.26} & \textcolor{posgreen}{+13.31} & \textcolor{posgreen}{+1.37} & \textcolor{posgreen}{+6.93} \\

    \midrule
    Qwen3-VL-2B\cite{qwen3} & 40.54 & 75.91 & 16.76 & 72.20 & 51.35 \\
    \rowcolor{highlightblue} w/ \textbf{DCI} & 51.36 & 76.03 & 20.13 & 75.78 & 55.83 \\
    \rowcolor{highlightblue} $\Delta$ & \textcolor{posgreen}{+10.82} & \textcolor{posgreen}{+0.12} & \textcolor{posgreen}{+3.37} & \textcolor{posgreen}{+3.58} & \textcolor{posgreen}{+4.48} \\

    \midrule
    Qwen3-VL-4B\cite{qwen3} & 64.81 & 79.71 & 34.54 & 85.25 & 66.08 \\
    \rowcolor{highlightblue} w/ \textbf{DCI} & 69.09 & 80.76 & 38.05 & 88.49 & 69.10 \\
    \rowcolor{highlightblue} $\Delta$ & \textcolor{posgreen}{+4.28} & \textcolor{posgreen}{+1.05} & \textcolor{posgreen}{+3.51} & \textcolor{posgreen}{+3.24} & \textcolor{posgreen}{+3.02} \\

    \midrule
    Qwen3-VL-8B\cite{qwen3} & 67.56 & 78.72 & 40.97 & 83.83 & 67.77 \\
    \rowcolor{highlightblue} w/ \textbf{DCI} & 71.16 & 81.63 & 46.85 & 85.65 & 71.32 \\
    \rowcolor{highlightblue} $\Delta$ & \textcolor{posgreen}{+3.60} & \textcolor{posgreen}{+2.91} & \textcolor{posgreen}{+5.88} & \textcolor{posgreen}{+1.82} & \textcolor{posgreen}{+3.55} \\
    \bottomrule
  \end{tabular}
  }
\end{table}
\textbf{Overall Performance Gains.} As summarized in Table~\ref{tab:benchmark}, the integration of DCI yields consistent and significant performance improvements across all evaluated MLLMs. Most notably, when applied to the Qwen2.5-VL-7B\cite{qwen2} architecture, DCI achieves a substantial average gain of +6.93\%, elevating the mean accuracy from 59.57\% to 66.50\%. These empirical results substantiate that the inherent discriminative capacity of these models is primarily bottlenecked by the flattened candidate space—which induces severe attention dilution—rather than a fundamental deficiency in visual representation.

\textbf{Task Adaptability and Scalability.} Beyond standard classification benchmarks, DCI demonstrates robust adaptability to fine-grained visual recognition tasks. For instance, it yields a +13.31\% accuracy improvement on the CUB-200\cite{cub200} dataset when integrated with Qwen2.5-VL-7B\cite{qwen2}. Furthermore, the efficacy of the proposed strategy remains stable across varying model parameter scales. From the lightweight Qwen3-VL-2B\cite{qwen3} (mean gain of +4.48\%) to the larger Qwen3-VL-8B\cite{qwen3} (mean gain of +3.55\%), the framework consistently enhances the models' visual discrimination capabilities. Notably, Qwen3-VL-8B\cite{qwen3} equipped with DCI achieves a mean accuracy of 71.32\%, successfully outperforming the proprietary state-of-the-art model GPT-4(v4.1) (70.43\%).

\subsection{Accuracy Gains in Large-Scale Visual Classification Task}

\begin{table}[!th]
\centering
\caption{Experimental results on the Imagenet-21K\cite{imagenet} dataset, comparing performance with and without the proposed Divide-and-Conquer Inference (DCI) strategy (where the DCI strategy is configured with $K=100$). $\Delta$ denotes the absolute performance improvement (\%). Bold numbers indicate the best performance in each column.}
\label{tab:imagenet21k_optimized}

\begin{tabular*}{\textwidth}{l @{\extracolsep{\fill}} ccc}
\toprule
\textbf{Model} & \textbf{Baseline (\%)} & \textbf{w/ DCI (\%)} & \textbf{$\Delta$ (\%)} \\
\midrule
GPT-4(v4.1)\cite{gpt4}      & 36.36          & 45.10          & +8.74  \\
Claude Opus 4.5\cite{claude}  & \textbf{45.48} & \textbf{61.64} & +16.16 \\
Qwen3-VL-Plus\cite{qwen3}    & 10.56          & 42.20          & +31.64 \\
Kimi-K2.5\cite{kimi}        & 15.82          & 32.26          & +16.44 \\
LLaMA-4\cite{llama}          & 11.14          & 14.28          & +3.14  \\
\midrule
Qwen2.5-VL-7B\cite{qwen2}    & \phantom{0}0.21 & 37.87          & \textbf{+37.66} \\
Qwen3-VL-2B\cite{qwen3}      & \phantom{0}0.48 & 19.29          & +18.81 \\
Qwen3-VL-4B\cite{qwen3}      & \phantom{0}0.96 & 34.41          & +33.45 \\
Qwen3-VL-8B\cite{qwen3}      & 10.56          & 34.20          & +23.64 \\
\bottomrule
\end{tabular*}
\end{table}

To assess the scalability of DCI in massive-scale visual recognition, benchmarks are conducted on the ImageNet-21K\cite{imagenet} dataset utilizing both open-source regimes (the Qwen family\cite{qwen2,qwen3}) and leading closed-source paradigms (e.g., GPT-4\cite{gpt4}, Claude Opus 4.5\cite{claude}, and Kimi-K2.5\cite{kimi}). 
Considering the economic and computational overhead associated with commercial closed-source APIs---especially since DCI expands the inference-time compute via multi-turn reasoning per instance---we adopt a bifurcated evaluation strategy. 
Specifically, a randomly sampled subset of 5,000 images is dedicated to the closed-source assessment, whereas a comprehensive validation set of 50,000 images is deployed for the open-source models.

The empirical results, detailed in Table~\ref{tab:imagenet21k_optimized}, reveal that DCI provides consistent and substantial accuracy improvements regardless of the baseline architecture. 
Remarkably, DCI elevates the classification performance of the lightweight Qwen2.5-VL-7B\cite{qwen2} by an absolute margin of 37.66\%, underlining its capability to counteract the attention dilution problem without requiring explicit parameter optimization.

\subsection{Speed Gains in Large-Scale Visual Classification Task}

To evaluate the computational efficiency of the proposed Divide-and-Conquer Inference (DCI) strategy, we analyze the inference latency alongside classification accuracy across varying scaling configurations of $K$. Table~\ref{tab:qwen3_imagenet21k_results} details the performance of the Qwen3-VL-2B model on the ImageNet-21K dataset under different configuration scales.

\begin{table}[!ht]
\centering
\caption{Performance evaluation of the Qwen3-VL-2B model on the ImageNet-21K dataset across distinct scaling configurations of $K$. Cell shading intensity represents performance quality: darker blue indicates higher accuracy ($\uparrow$), while darker orange indicates lower latency ($\downarrow$).}
\label{tab:qwen3_imagenet21k_results}
\small
\resizebox{\linewidth}{!}{%
\begin{tabular}{lcccccc}
\toprule
\textbf{Metric} & \textbf{$K=50$} & \textbf{$K=100$} & \textbf{$K=500$} & \textbf{$K=1000$} & \textbf{$K=5000$} & Baseline \\
\midrule
Accuracy (\%) $\uparrow$ &\cca{10} 16.12 &\cca{25} 19.29 &\cca{15} 18.86 &\cca{12} 17.72 &\cca{8} 13.08  &\cca{2} 0.48\\
Latency (s) $\downarrow$ & \ccl{4}13.49 & \ccl{6}10.55 & \ccl{25}5.99  & \ccl{17} 6.00  & \ccl{15} 6.74    &\ccl{10} 7.96\\
\bottomrule
\end{tabular}
}
\end{table}

Contrary to the intuitive assumption that multi-turn reasoning inherently introduces severe computational overhead, our empirical findings reveal that optimal partitioning can simultaneously improve inference efficiency and substantially enhance recognition performance. As discussed in Section~\ref{sec:ceaca}, DCI demonstrates a unique computational advantage on large-scale datasets, effectively breaking the conventional trade-off between accuracy and efficiency. Specifically, configuring the DCI strategy with $K=500$ reduces the inference latency to 5.99 seconds---yielding a notable acceleration over the 7.96-second baseline---while preserving a robust accuracy of 18.86\%, representing a substantial improvement over the 0.48\% baseline.

This counterintuitive speed gain can be largely attributed to the mitigation of the self-attention mechanism's quadratic complexity. By dividing the massive 21K class list into smaller subsets governed by $K$, the context length per forward pass is significantly reduced. While an extremely small scale (e.g., $K=50$) requires excessive independent queries, causing the multi-turn overhead to dominate and spiking latency to 13.49 seconds, an appropriately sized $K$ (such as $K=500$ or $K=1000$) strikes an optimal balance. It reduces the per-prompt sequence length enough to process faster than the monolithic baseline, without incurring excessive structural overhead. 

Furthermore, the results reveal a clear performance-efficiency trade-off. The configuration $K=100$ yields the peak classification accuracy at 19.29\%, albeit with a moderate latency increase to 10.55 seconds. Conversely, overly large partitions ($K=5000$) begin to approach the baseline's attention dilution problem, suffering from both a sharp drop in accuracy (13.08\%) and a latency rebound (6.74 seconds). These findings underscore the flexibility of DCI, allowing practitioners to dynamically adjust $K$ based on their specific latency and precision constraints.

\subsection{Impact of Group Size K}

To thoroughly investigate the influence of the hyperparameter K (the local group size during the divide-and-conquer process), we conduct an ablation study utilizing the Qwen3-VL-8B\cite{qwen3} architecture. Table~\ref{tab:ik} presents the classification accuracy on the CIFAR100\cite{cifar100} and ImageNet-1K\cite{imagenet} datasets across varying group sizes ($K \in \{2, 5, 10, 20, 50\}$), alongside the standard monolithic inference baseline (denoted as "-").

\begin{table}[htp!]
\centering
\caption{Performance and efficiency trade-off of Qwen3-VL-8B across different group sizes $K$. Accuracy (\%) and latency (s) are reported for CIFAR100 and ImageNet-1K. Cell shading intensity represents performance quality: darker blue indicates higher accuracy ($\uparrow$), while darker orange indicates lower latency ($\downarrow$). The DCI framework consistently outperforms the baseline ("-") in accuracy.}\label{tab:ik}
\setlength{\tabcolsep}{5pt}
\renewcommand{\arraystretch}{1.3} 
\resizebox{\linewidth}{!}{%
\begin{tabular}{llcccccc}
\toprule
Dataset & Metric & $K=2$ & $K=5$ & $K=10$ & $K=20$ & $K=50$ & - \\
\midrule
\multirow{2}{*}{CIFAR100} & Acc. (\%) $\uparrow$ & \cca{40} 81.36 & \cca{45} 81.62 & \cca{50} 81.63 & \cca{30} 81.00 & \cca{20} 80.27 & \cca{10} 78.72 \\
 & Lat. (s) $\downarrow$ & \ccl{5} 1.14 & \ccl{15} 0.53 & \ccl{30} 0.37 & \ccl{45} 0.32 & \ccl{60} 0.26 & \ccl{70} 0.20 \\
\midrule
\multirow{2}{*}{ImageNet-1K} & Acc. (\%) $\uparrow$ & \cca{35} 69.29 & \cca{45} 70.92 & \cca{55} 71.16 & \cca{45} 70.99 & \cca{35} 70.14 & \cca{10} 67.56 \\
 & Lat. (s) $\downarrow$ & \ccl{5} 9.06 & \ccl{20} 3.19 & \ccl{40} 1.91 & \ccl{55} 1.63 & \ccl{75} 1.32 & \ccl{85} 0.71 \\
\bottomrule
\end{tabular}
}
\end{table}

Several critical observations can be drawn from these results. First, our DCI framework consistently outperforms the baseline across all evaluated values of $K$, underscoring the inherent robustness of the divide-and-conquer strategy. Second, the performance exhibits a distinct inverted U-shaped trend with respect to the group size. In contrast, inference latency demonstrates a monotonic inverse relationship with $K$, as larger group sizes effectively reduce the reasoning depth and the number of sequential inference stages. As theoretically analyzed in Section~\ref{sec:ceaca}, when the label space remains relatively small (e.g., $N=100$ or $N=1000$), the computational complexity decreases with increasing $K$, which is consistent with our empirical observations. The optimal efficacy is achieved at $K=10$ for both benchmarks, yielding 81.63\% on CIFAR100\cite{cifar100} and 71.16\% on ImageNet-1K\cite{imagenet}.

This non-monotonic behavior highlights a fundamental trade-off governed by the grouping mechanism. Specifically, $K$ balances the reasoning depth against local discriminative power: a smaller $K$ increases depth for fine-grained comparisons but incurs a significant "latency tax," while a larger $K$ prioritizes speed at the risk of context dilution. When K is excessively small (e.g., K = 2), the hierarchical reasoning tree becomes overly deep. This extended multi-stage inference not only risks error propagation across levels but also deprives the model of sufficient comparative context among sibling categories. Conversely, as K increases to larger values (e.g., K = 50), the local candidate space dramatically expands. This expansion gradually reintroduces the severe attention dilution and information collapse that plague the standard flat classification approach. Therefore, K = 10 strikes an optimal balance—maintaining a high local Signal-to-Noise Ratio (SNR) while keeping the reasoning hierarchy sufficiently compact.

\subsection{Comparison with Test-time Scaling Strategies}

To evaluate the superiority and practicality of the DCI framework, we perform a comprehensive comparison against representative test-time scaling and reasoning strategies, including Chain-of-Thought (CoT)\cite{cot}, Self-Consistency (SC)\cite{sc}. Although advanced reasoning paradigms such as ToT~\cite{tot} and GoT~\cite{got} have shown promising performance in open-ended reasoning tasks, their original formulations are not directly applicable to large-scale classification and require substantial task-specific adaptation, hindering fair comparisons. Therefore, our experiments focus on representative classification-oriented baselines. All experiments rely on the Qwen3-VL-8B\cite{qwen3} architecture across CUB-200\cite{cub200}, ImageNet-1K\cite{imagenet} and ImageNet-21K\cite{imagenet}.

\begin{table}[!h]
  \centering
  \caption{Comprehensive scalability analysis of DCI. Acc. (\%) and Lat. (s) denote Top-1 accuracy and inference latency, respectively. \textbf{Improv.} rows highlight the accuracy gain (\textbf{+}) and latency ratio (\textbf{$\times$}) relative to the Baseline. Results are reported as the best values over multiple independent runs. The branching factor is configured according to the task scale, with $K=10$ for CUB-200 and ImageNet-1K, and $K=100$ for ImageNet-21K.}
  \label{tab:tts}
  \setlength{\tabcolsep}{6pt} 
  
  \resizebox{\linewidth}{!}{%
    \begin{tabular}{l cc c cc c cc}
      \toprule
      \multirow{2}{*}{\textbf{Method}} & \multicolumn{2}{c}{\textbf{CUB-200\cite{cub200}}} & & \multicolumn{2}{c}{\textbf{ImageNet-1K\cite{imagenet}}} & & \multicolumn{2}{c}{\textbf{ImageNet-21K\cite{imagenet}}} \\
      \cmidrule{2-3} \cmidrule{5-6} \cmidrule{8-9}
      & Acc. $\uparrow$ & Lat. $\downarrow$ & & Acc. $\uparrow$ & Lat. $\downarrow$ & & Acc. $\uparrow$ & Lat. $\downarrow$ \\
      \midrule
      Baseline & 40.97 & 0.49 & & 67.56 & 1.05 & & 10.56 & 20.67 \\
      CoT\cite{cot} & 41.59 & 0.62 & & 67.75 & 1.53 & & 11.38 & 24.15 \\
      SC\cite{sc} & 39.47 & 2.49 & & 67.29 & 5.02 & & 10.98 & 103.5 \\
      \midrule
      \addlinespace[0.8ex]
      
      \rowcolor[rgb]{0.93, 0.96, 1.0} \multicolumn{9}{l}{\textit{DCI (Sequential Mode)}} \\
      \rowcolor[rgb]{0.93, 0.96, 1.0} DCI (Ours) & 46.85 & 2.15 & & 71.16 & 3.37 & & 48.92 & 44.63 \\
      \rowcolor[rgb]{0.93, 0.96, 1.0} \textbf{Improv. vs. Base} & +5.88 & 4.39$\times$ & & +3.60 & 3.21$\times$ & & +38.36 & 2.16$\times$ \\
      \addlinespace[0.8ex]
      
      \rowcolor[rgb]{1.0, 0.97, 0.9} \multicolumn{9}{l}{\textit{DCI (Parallel Mode)}} \\
      \rowcolor[rgb]{1.0, 0.97, 0.9} + Parallel & 46.85 & 1.42 & & 71.16 & 1.91 & & 48.92 & 27.23 \\
      \rowcolor[rgb]{1.0, 0.97, 0.9} \textbf{Improv. vs. Base} & +5.88 & 2.90$\times$ & & +3.60 & 1.82$\times$ & & +38.36 & 1.32$\times$ \\
      \bottomrule
    \end{tabular}%
  }
\end{table}

\textbf{Task Suitability and Accuracy.}
As shown in Table~\ref{tab:tts}, the proposed DCI framework outperforms representative baselines, and this advantage becomes increasingly pronounced as the task scale $N$ grows. Specifically, DCI achieves an improvement of up to \textbf{+38.36\%} over the baseline, reaching a final accuracy of \textbf{48.92\%}. In contrast, reasoning-based approaches such as CoT\cite{cot} and SC\cite{sc} provide only marginal gains in image classification tasks. These results suggest that existing TTS methods are primarily designed for logical and mathematical reasoning scenarios and exhibit limited effectiveness in large-scale visual recognition. In comparison, DCI demonstrates substantially superior performance and is particularly effective for large-scale image classification tasks.

\textbf{Complexity and Computational Efficiency.}
As discussed in Section~\ref{sec:ceaca}, DCI exhibits increasing computational advantages as the scale of image classification tasks expands. As reported in Table~\ref{tab:tts} on the CUB-200\cite{cub200}, ImageNet-1K\cite{imagenet}, and ImageNet-21K\cite{imagenet} benchmarks, the relative latency overhead of DCI continuously decreases with increasing $N$. Specifically, under the serial execution setting, the relative latency overhead compared with the baseline model decreases from \textbf{4.39$\times$} ($N=200$) to \textbf{2.16$\times$} ($N=21{,}000$). This trend indicates that DCI possesses superior computational scalability and demonstrates increasingly favorable efficiency characteristics in ultra-large-scale classification scenarios.

\textbf{Parallel Scalability.}
Furthermore, under ideal conditions with fully parallel execution and negligible communication overhead, the theoretical computational complexity in Eq.~\ref{eq:complexity2} becomes
$T(N,K)\approx \left(\log_K N\right)\cdot K^2$
, leading to a substantial reduction in inference latency. As shown in Table~\ref{tab:tts}, on ImageNet-21K, the parallel implementation (\textbf{+ Parallel}) reduces the absolute inference time from \textbf{44.63 s} to \textbf{27.23 s}, compressing the relative latency overhead to only \textbf{1.32$\times$} compared with the baseline model. These results demonstrate that DCI offers strong practical potential as an efficient and highly scalable test-time inference strategy for real-world ultra-large-scale recognition applications.

\subsection{Impact of Category Grouping Strategies}

To examine the sensitivity of the DCI framework to different category partitioning methods, we evaluate three distinct grouping strategies on ImageNet-1K: (i) \textbf{Random grouping}, where categories follow a uniform random distribution; (ii) \textbf{Most similar}, where categories are clustered based on semantic embeddings to maximize intra-group coherence; and (iii) \textbf{Least similar}, which aims to maximize intra-group semantic diversity.

\begin{table}[!h]
\centering
\caption{Ablation study of different grouping strategies on ImageNet-1K[15]. This table evaluates the classification accuracy (in \%) using Qwen2.5-VL-7b\cite{qwen2} and Qwen3-VL-4B\cite{qwen3} across three category grouping methodologies. All experiments adopt a fixed group size of $K=10$ for DCI.}
\label{tab:gs}

\begin{tabular*}{\linewidth}{@{\extracolsep{\fill}}lcc}
\toprule
Grouping Strategy & Qwen2.5-VL-7b\cite{qwen2}(\%) & Qwen3-VL-4B\cite{qwen3}(\%)\\
\midrule
Random & 71.57 & 69.09\\
Most similar & 71.32 & 68.81\\
Least similar & 71.39 & 69.11\\
\bottomrule
\end{tabular*}
\end{table}

As reported in Table~\ref{tab:gs}, the performance variance across the three strategies is relatively small. For both Qwen2.5-VL-7b\cite{qwen2} and Qwen3-VL-4B\cite{qwen3}, the accuracy fluctuation remains within 0.3\%. For instance, when using Qwen2.5-VL-7b\cite{qwen2}, the Random strategy reaches 71.57\%, which is slightly higher than the 71.32\% achieved by the Most similar strategy. These data indicate that altering the grouping strategy does not significantly improve performance. For example, when we implement the Least similar strategy to maximize intra-group semantic diversity, the classification difficulty within that specific group decreases. However, the model then encounters extremely similar distractors in the subsequent level of reasoning, thereby offsetting the initial gains. These findings highlight that the core advantage of DCI stems from the reduction of the local candidate space and the mitigation of attention dilution, rather than the specific semantic organization of the groups.

\subsection{Mitigate Performance Collapse in Long Sequence Recognition}

To validate the efficacy of the proposed Divide-and-Conquer Inference (DCI) framework under varying levels of task complexity, we conduct a systematic comparative analysis across six representative vision-language models (MLLMs). By scaling the number of candidate classes ($N$) from 10 to 1,000, we observe several critical performance trends as illustrated in Fig. \ref{fig:performance}:

\begin{figure}[t!]
    \centering
    \includegraphics[width=1.0\linewidth]{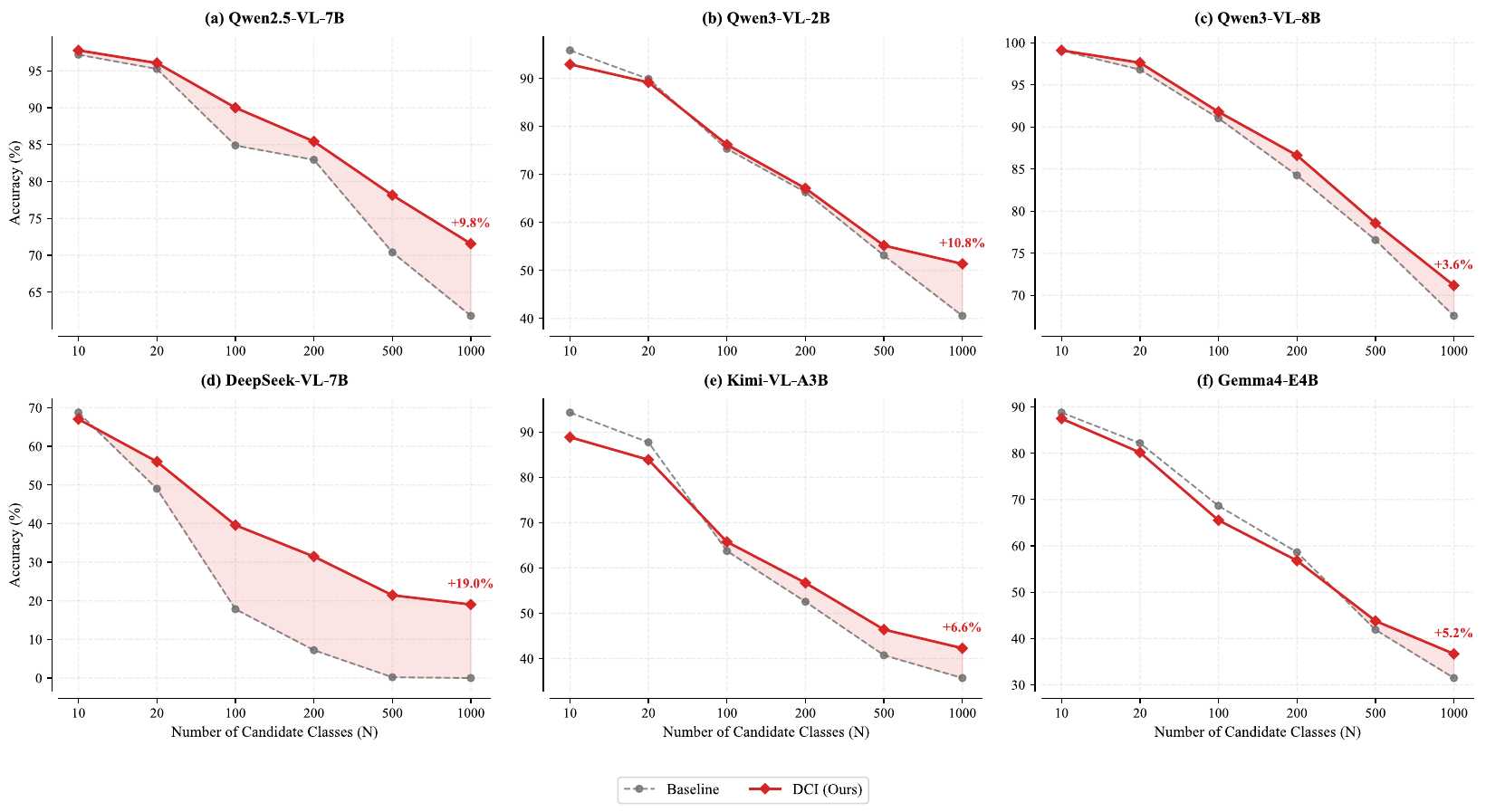}
    \caption{Empirical comparison between the proposed DCI framework and the baseline across diverse MLLM architectures on sampled subsets of the ImageNet-1K dataset with a fixed group size of $K=10$. The curves depict accuracy (\%) relative to the cardinality of candidate classes ($N$). DCI (red) consistently exhibits superior robustness, effectively mitigating performance degradation as the search space expands.}
    \label{fig:performance}
\end{figure}

\begin{itemize}
    \item \textbf{Suppression of Scale-Induced Degradation:} While standard flat inference suffers from a sharp accuracy decline as the label space expands, DCI (red line) significantly inhibits this degradation. By decomposing a single massive decision space into manageable sub-tasks, DCI preserves discriminative precision. Notably, for DeepSeek-VL-7B\cite{deepseek_vl}, where the baseline performance nearly collapses at $N=1,000$, our framework achieves a remarkable \textbf{$+19.0\%$} lead, demonstrating its robustness in high-cardinality scenarios.
    
    \item \textbf{Efficiency Transition in Small-Scale Tasks:} In low-complexity scenarios (e.g., $N \le 20$), the performance of DCI is comparable to the baseline. This observation aligns with our theoretical expectation: for tasks with minimal class competition, the inherent reasoning capacity of MLLMs is sufficient. However, as $N$ increases, the ``divide-and-conquer'' strategy becomes indispensable, progressively widening the performance gap over the standard approach.
    
    \item \textbf{Model-Agnostic Stability and Performance Buffer:} The capacity of DCI to mitigate accuracy decay generalizes across various model families and parameter scales. The shaded regions in Fig. \ref{fig:performance} represent the performance buffer provided by our hierarchical decomposition. Whether implemented on lightweight models like Qwen3-VL-2B\cite{qwen3} (\textbf{$+10.8\%$} gain) or larger architectures like Gemma4-E4B (\textbf{$+5.2\%$} gain), DCI consistently maintains stable accuracy.
\end{itemize}

In summary, the experimental results demonstrate that the DCI framework serves as a pivotal mechanism for inhibiting PC-LSR. As the cardinality of the label space expands, standard inference methods often suffer from critical failures due to cumulative logical noise and reasoning overloads. In contrast, by systematically decomposing the global classification task into a series of localized, manageable sub-decisions, DCI effectively stabilizes the model's discriminative power. This robustness transforms high-difficulty, large-scale recognition challenges into a reliable inference process, preventing the catastrophic accuracy decay typically observed in long-sequence candidate scenarios.

\section{Discussion}

\subsection{The Essence of PC-LSR: Information Entropy vs. Attention Dilution}
The experimental results validate our hypothesis on Performance Collapse in Long Sequence Recognition. Unlike traditional vision models with fixed label spaces, MLLMs process category information within limited context windows. As the number of categories $N$ increases, the required information entropy grows, while attention mechanisms suffer from inherent dilution effects over long candidate sequences. Consequently, relevant category signals become increasingly attenuated. By decomposing the global task into local subproblems, DCI effectively reduces entropy at each stage, maintaining a higher signal-to-noise ratio during inference.

\subsection{DCI as a Test-Time Scaling Paradigm}
Recently, the focus of the AI community has shifted from training-time scaling to test-time scaling (e.g., OpenAI's o1\cite{o1}). DCI aligns with this paradigm by demonstrating that \textbf{inference-time compute can be traded for classification accuracy}. While traditional MLLM inference is a "one-shot" process, DCI introduces a structured, multi-step reasoning path. Our results on ImageNet-21K\cite{imagenet} suggest that the performance bottleneck for lightweight models on large-scale tasks may stem from inference-time structural limitations rather than a fundamental lack of visual knowledge. DCI acts as a bridge that unlocks this latent capability without any additional training.

\subsection{Generality and Model-Agnostic Nature}
A key advantage of DCI is its model-agnostic nature. Whether applied to proprietary models like GPT-4\cite{gpt4} or open-source models like Qwen-VL\cite{qwen3}, DCI generally yields performance gains. This suggests that the attention dilution problem is a fundamental characteristic of the Transformer architecture rather than a flaw of specific model weights. Furthermore, in large-scale image classification scenarios, DCI achieves computational complexity that scales more favorably than quadratic growth, making it not only more accurate but also more robust for extreme-scale classification tasks.

\subsection{Limitations and Trade-offs}
Despite its effectiveness, DCI introduces a trade-off between inference latency and accuracy in small-scale datasets, as recursive reasoning inevitably increases the overall wall-clock time compared to single-pass inference. However, in many high-stakes scenarios, such as medical imaging\cite{medicine}, the gain in precision far outweighs the incremental increase in compute time. Future work will explore parallelization strategies to further optimize the throughput of DCI.

\section{Conclusion}

In this study, we investigate the performance degradation of MLLMs in large-scale recognition tasks. By identifying and formalizing the phenomenon of Performance Collapse in Long Sequence Recognition, we provide a new perspective for understanding the limitations of current Transformer-based multimodal architectures. Our
information-theoretic analysis offers a rigorous explanation for this collapse, revealing that it stems from inherent attention dilution and signal decay when the label space exceeds the model’s effective signal-to-noise ratio threshold.

To address these structural limitations, we propose Divide-and-Conquer Inference, a pure test-time scaling strategy that requires neither retraining nor architectural modifications. By reformulating inference as a recursive hierarchical search process, DCI effectively mitigates attention dilution while substantially reducing the computational burden associated with exhaustive large-scale classification. Experimental results on challenging benchmarks, including ImageNet-21K\cite{imagenet}, demonstrate that DCI not only restores classification accuracy but also accelerates inference in large-scale image classification tasks, enabling lightweight open-source models to achieve competitive performance with significantly larger closed-source systems. 

The significance of DCI lies in its model-agnostic, training-free, and computationally efficient nature, offering a practical paradigm for scaling MLLM capabilities without incurring the prohibitive costs of retraining. This work opens a promising direction toward inference-time compute optimization for multimodal reasoning and large-scale
recognition tasks.

\bibliographystyle{elsarticle-num} 
\bibliography{main}



\end{document}